\documentclass[a4paper,fleqn]{cas-dc}

\usepackage[numbers]{natbib}
\usepackage{hyperref}
\usepackage{url}

\usepackage{mathtools}  % Loads amsmath automatically
\usepackage{amsfonts}
\usepackage{amssymb}
\usepackage{amsthm}
\usepackage{bm}
\usepackage{siunitx}
\usepackage{pifont}     % For checkmarks/crosses (\ding)

\usepackage{tabularx}   % Required for tabularx errors
\usepackage{booktabs}   % For professional quality tables
\usepackage{adjustbox}  % For resizing tables
\usepackage{multirow}   % For multi-row cells
\usepackage{makecell}   % For multi-line cells
\usepackage{threeparttable}
\usepackage{graphicx}
\usepackage{subcaption} 
\usepackage{caption}
\usepackage{float}
\usepackage{wrapfig}

\usepackage[utf8]{inputenc}
\usepackage{soul}       
\usepackage[normalem]{ulem} 
\usepackage{xspace}     
\usepackage{enumitem}   
\usepackage[dvipsnames]{xcolor} % Extended color names

\usepackage{newunicodechar}
\newunicodechar{≥}{\geq} 

\usepackage{listings}
\usepackage{minted}
\usepackage{algorithm}
\usepackage{etoolbox}

\lstdefinestyle{mystyle}{
    backgroundcolor=\color{white},
    commentstyle=\color{black},
    keywordstyle=\color{teal},
    numberstyle=\tiny\color{gray},
    stringstyle=\color{BrickRed},
    basicstyle=\ttfamily\footnotesize,
    breaklines=true,
    numbers=none,
    frame=single,
    showspaces=false,
    showstringspaces=false,
    showtabs=false,
    tabsize=2,
    morekeywords={Task, Decomposition, KPM, Scores}
}

\lstdefinestyle{verifycode}{
  language=Python,
  basicstyle=\linespread{0.80}\ttfamily\bfseries\footnotesize,
  keywordstyle=\color{blue!99!black}\bfseries,
  commentstyle=\color{gray!99},
  showstringspaces=false,
  breaklines=true,
  breakatwhitespace=true,
  tabsize=2,
  numbers=none,
  frame=none,
  backgroundcolor=\color{white},
  aboveskip=0.1em, belowskip=0.1em,
  xleftmargin=0em, xrightmargin=0em
}

\usepackage{tikz}
\usepackage[title]{appendix}

\theoremstyle{definition}
\newtheorem{definition}{Definition}
\newtheorem{theorem}{Theorem}
\newtheorem{lemma}{Lemma}

\AtBeginEnvironment{definition}{\setlength{\parindent}{0pt}}
\AtBeginEnvironment{theorem}{\setlength{\parindent}{0pt}}
\AtBeginEnvironment{lemma}{\setlength{\parindent}{0pt}}

\newcommand{\rectgreen}[1]{%
  \tikz[baseline=(char.base)]{
    \node[shape=rectangle,rounded corners=2pt,fill=ForestGreen,inner sep=1.5pt,minimum height=1.2em] (char) {\textcolor{white}{#1}};
  }%
}

\newcommand*\circledgreen[1]{\tikz[baseline=(char.base)]{
    \node[shape=circle, fill=ForestGreen, inner sep=1.3pt] (char) {\textcolor{white}{\textbf{#1}}};}}

\begin{document}

\let\WriteBookmarks\relax
\def\floatpagepagefraction{1}
\def\textpagefraction{.001}

% Short title
\shorttitle{Verify-RL}

% Short author
\shortauthors{Qasim et~al.}

% Main title of the paper
\title [mode = title]{VERIFY-RL: Verifiable Recursive Decomposition for Reinforcement Learning in Mathematical Reasoning}                      
% Title footnote mark
% \tnotemark[1] 

% First author

\author[aff1]{Kaleem Ullah Qasim}
\ead{kaleem@my.swjtu.edu.cn}

% Second author
\author[aff1]{Jiashu Zhang}
\cormark[1] % Corresponding author indication
\ead{jszhang@home.swjtu.edu.cn}
\cortext[1]{Corresponding author}
 
%Third
\author[aff1]{Hao Li}
\ead{hli@my.swjtu.edu.cn}

\author[aff1]{Muhammad Kafeel Shaheen}
\ead{kafeel@my.swjtu.edu.cn}

% Affiliations
\affiliation[aff1]{organization={School of Computing and Artificial Intelligence, Southwest Jiaotong University}, 
                   city={Chengdu}, 
                   postcode={611756}, 
                   country={China}}

\credit{Data curation, Writing - Original draft preparation}

\begin{abstract}
Training language models to solve complex mathematical problems benefits from curriculum learning progressively training on simpler subproblems. However, existing decomposition methods are often heuristic, offering no guarantees that subproblems are simpler, that solving them aids the parent task, or that their relationships are mathematically grounded. We observe that symbolic differentiation provides a natural structure for verified decomposition: calculus rules explicitly define how expressions reduce to simpler components with provable properties. We introduce \textit{Verify-RL}, a framework where every parent-child decomposition satisfies three verifiable conditions: strictly decreasing structural complexity, solution containment, and formal rule derivation. Unlike heuristic methods where a significant fraction of decompositions are invalid our properties admit automatic verification through symbolic computation, achieving "verification by construction." Experiments demonstrate that eliminating invalid decompositions yields sizable gains, accuracy on the hardest problems more than doubles from 32\% to 68\%, with a 40\% relative improvement overall.
\end{abstract}


\maketitle
\section{Introduction}
\label{sec:introduction}

Curriculum learning improves mathematical reasoning in language models by decomposing complex problems into simpler subproblems~\cite{cot,gsm8k}. Recent work shows that reinforcement learning with verifiable rewards can push models beyond their supervised training distribution~\cite{deepseek-r1}, and that ordering problems by difficulty improves both sample efficiency and final performance~\cite{e2h2025_curriculum,dwl-curriculum}. The core challenge lies in constructing such curricula: how do we ensure that subproblems are genuinely easier, that solving them helps solve the original, and that parent-child relationships are mathematically meaningful?

Current decomposition methods face a fundamental limitation. Heuristic approaches rely on language models to propose easier variants of hard problems~\cite{ladder}, offering flexibility but no guarantees about validity. When evaluated against formal criteria, a significant fraction of generated decompositions fail basic properties: the child is not actually easier, its solution does not help solve the parent, or the relationship is spurious. These invalid examples introduce noise into training. Formal verification in theorem proving~\cite{deepseekprover2024,leancop2024} provides provable correctness, but such methods have remained confined to narrow domains with established proof assistants.

We present \textsc{Verify-RL}, a framework that achieves the flexibility of curriculum learning with the guarantees of formal verification. The key insight is that symbolic differentiation possesses decomposition rules with built-in mathematical structure: calculus rules define exactly how expressions reduce to simpler components. When differentiating $\sin(x^2)$, the chain rule specifies that $x^2$ becomes a subproblem, that its derivative appears in the parent's solution, and that this relationship follows from a formal rule. We formalize this observation through three verifiable properties: the child must be structurally simpler, the child's solution must appear in the parent's solution, and the relationship must derive from a calculus rule. These properties admit automatic verification through symbolic computation, achieving complete verification by construction.

We train models ranging from 0.6B to 3B parameters using multiple RL algorithms on curricula constructed via verified decomposition. Training progresses from base-case derivatives through compositions of increasing depth, with each problem's decomposition tree guaranteeing prerequisite ordering. Unlike heuristic methods where a significant fraction of decompositions are invalid, our rule-based approach achieves complete verification by construction. This translates to substantial accuracy gains: performance more than doubles on the hardest problems, with 40\% relative improvement overall. Our contributions are:

\vspace{0.3em}
\noindent
\textbf{\circledgreen{1}} A formal framework for verifiable decomposition with three properties structural simplicity, solution containment, and rule derivation that admit automatic verification through symbolic computation.

\vspace{0.3em}
\noindent
\textbf{\circledgreen{2}} Decomposition operators for chain, product, and sum rules with proofs that each preserves verification properties, achieving complete verification by construction.

\vspace{0.3em}
\noindent
\textbf{\circledgreen{3}} A curriculum construction algorithm that guarantees prerequisite ordering through topological sorting of verified decomposition trees.

\vspace{0.3em}
\noindent
\textbf{\circledgreen{4}} Experiments across multiple model scales demonstrating substantial accuracy gains, with systematic ablations isolating each component's contribution.

\begin{figure*}
    \centering
    \includegraphics[width=1\linewidth]{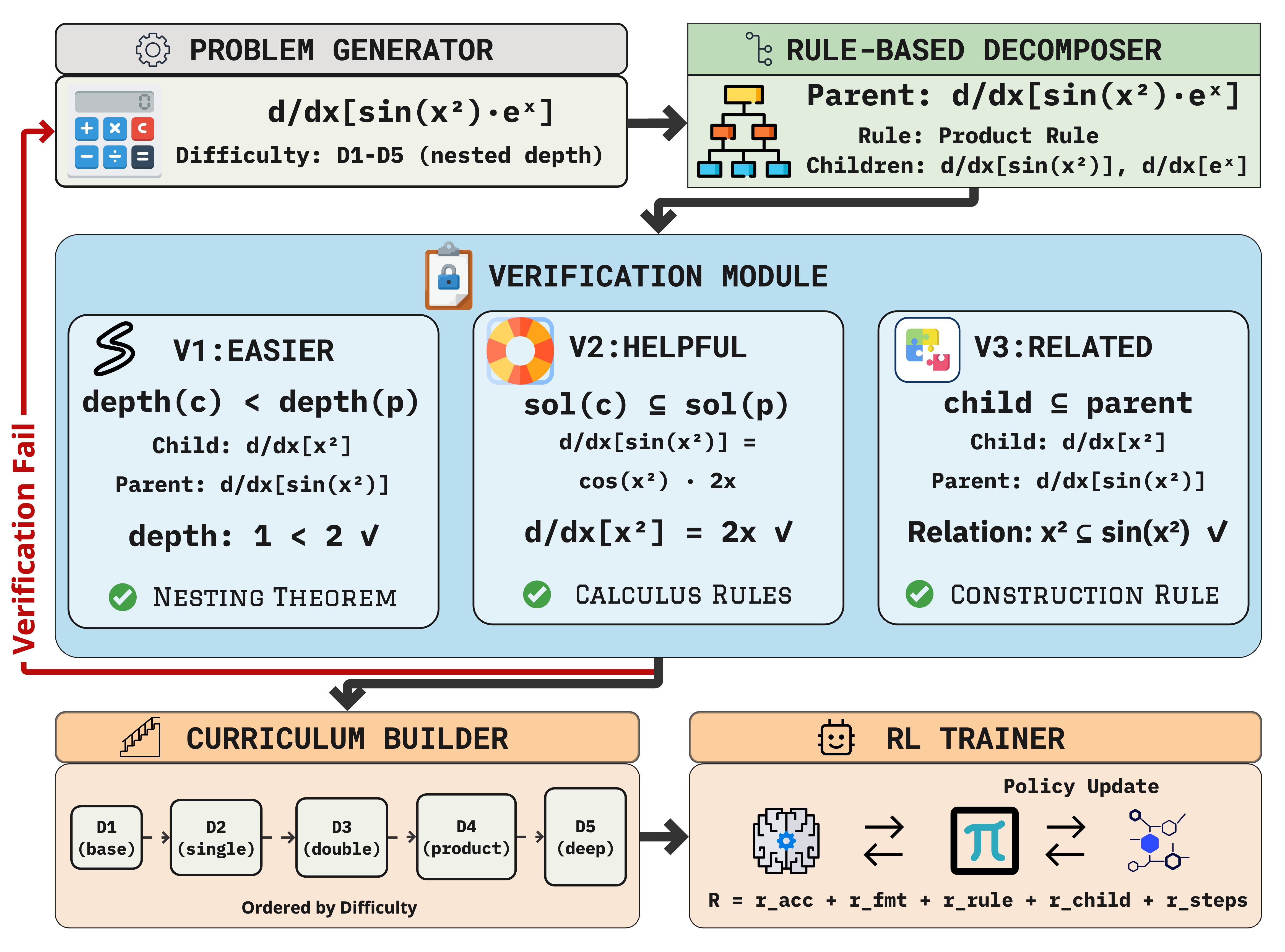}
    \caption{\textsc{Verify-RL} framework. Problems pass through rule-based decomposition, verification (V1\textendash V3), and curriculum ordering before RL training. Failed verifications trigger re-decomposition.}
    \label{fig:framework}
\end{figure*}
The rest of this paper is organized as follows: Section~\ref{sec:related_work} reviews related work on curriculum learning and reasoning verification. Section~\ref{sec:method} presents our formal framework including problem formulation, decomposition theory, and verification procedures. Section~\ref{sec:experiments} describes experimental setup and research questions. Section~\ref{sec:results} presents main results and training dynamics. Section~\ref{sec:ablation} provides systematic ablations isolating each component's contribution. Section~\ref{sec:conclusion} concludes with limitations and future directions.

\section{Related Work}
\label{sec:related_work}

Mathematical reasoning in language models has advanced through prompting strategies that encourage structured problem-solving. Chain-of-thought prompting~\cite{cot} demonstrated that eliciting intermediate steps improves arithmetic accuracy, while self-consistency~\cite{self-consistency} aggregates multiple reasoning paths via majority voting to reduce variance. These methods expose reasoning but do not decompose problems into subproblems with guaranteed relationships.

Task decomposition approaches address complex reasoning by breaking problems into simpler components. Least-to-most prompting~\cite{least-to-most} teaches models to solve subproblems sequentially from easiest to hardest, achieving strong generalization on symbolic manipulation tasks. Tree of Thoughts~\cite{tot} extends chain-of-thought by exploring multiple reasoning branches with backtracking, enabling deliberate planning. ADaPT~\cite{adapt} dynamically decomposes failed tasks into subtasks based on execution feedback. FunCoder~\cite{funcoder} applies recursive decomposition to code generation. Recent work on complexity-agnostic recursive decomposition~\cite{cardt} and logical thought decomposition~\cite{rdolt} explores frameworks for breaking down reasoning tasks recursively, demonstrating improved performance through structured decomposition and knowledge propagation. Hierarchical multi-agent methods~\cite{hisoma-decomposition} have shown that temporal and structural task decomposition enables effective learning for complex long-horizon problems, though decomposition strategies remain domain-specific. These methods rely on the model to generate or evaluate decompositions heuristically and the subproblem-parent relationship is assumed rather than verified.

Self-improvement enables models to bootstrap reasoning capability from their own outputs. STaR~\cite{star} iteratively fine-tunes on model-generated rationales that lead to correct answers, demonstrating that models can learn from successful reasoning traces. ReST~\cite{rest-em} scales self-training by filtering generated solutions through verifiers, while V-STaR~\cite{v-star} trains verifiers alongside generators to improve sample efficiency. ReST-MCTS*~\cite{rest-mcts} combines process reward guidance with tree search for more targeted self-improvement. These approaches assume access to correctness signals but do not exploit structured problem decomposition.

Reinforcement learning has emerged as a powerful paradigm for reasoning improvement. RLHF~\cite{rlhf} aligns models with human preferences through learned reward functions, while PPO~\cite{ppo} provides stable policy updates. DeepSeek-R1~\cite{deepseek-r1} demonstrated that pure RL with verifiable rewards (RLVR) can induce sophisticated reasoning patterns, introducing GRPO for efficient group-relative optimization. Hierarchical RL with option scheduling~\cite{uos-hierarchical-rl} addresses sparse reward challenges through unlimited option abstraction. VinePPO~\cite{vineppo} addresses credit assignment limitations in standard PPO by computing Monte Carlo estimates of intermediate values. Recent work~\cite{rlvr} shows RLVR extends model capability beyond supervised fine-tuning. GRPO-LEAD~\cite{grpo-lead} incorporates difficulty-aware scheduling into GRPO training. DeepSeekMath~\cite{deepseekmath} pushed mathematical reasoning limits by combining large-scale pretraining with GRPO fine-tuning.

Process reward models (PRMs) provide step-level supervision for reasoning verification. Lightman et al.~\cite{prm} found that process supervision outperforms outcome supervision for mathematical reasoning, achieving 78\% accuracy on MATH. Math-Shepherd~\cite{math-shepherd} automates process annotation without human labels by verifying step correctness through solution sampling. The distinction between outcome and process supervision~\cite{orm-vs-prm} informs reward design: outcome rewards verify final answers while process rewards assess intermediate steps.

Curriculum learning~\cite{curriculum} organizes training from easy to hard examples to improve convergence. Recent work~\cite{e2h-curriculum} applies this principle to LLM reasoning, where difficulty-ordered training improves mathematical performance. Domain-specific models like Minerva~\cite{minerva}, Llemma~\cite{llemma} and MathCoder~\cite{mathcoder} combine mathematical pretraining with structured fine-tuning.

LADDER~\cite{ladder} introduced recursive decomposition for self-improvement, where models generate easier variants of hard problems to create training curricula. The model trains on solvable subproblems then bootstraps to harder ones. Complementary approaches~\cite{cardt,rdolt} explore recursive decomposition through complexity-agnostic frameworks and logical reasoning structures, though these methods similarly employ model-generated decompositions. However, LADDER and related heuristic approaches rely on decomposition without mathematical guarantees: the model generates ``easier'' variants that may not satisfy formal properties. Under formal verification (V1\textendash V3), 21.6\% of generated subproblems fail at least one property, introducing noise into training.

Despite these advances, existing approaches share three fundamental limitations: \textbf{(1)~Heuristic decomposition}   methods rely on models to generate subproblems without guarantees that children are actually easier or helpful; \textbf{(2)~No formal verification}   curriculum ordering assumes rather than verifies prerequisite relationships; and \textbf{(3)~Training instability}   invalid decompositions introduce noise, causing models to encounter unsolvable problems or learn spurious patterns. Our work addresses these limitations through rule-based decomposition that provides mathematical guarantees (V1\textendash V3), achieving 100\% verification success compared to 78.4\% for model-generated approaches.

\begin{table*}
\centering
\caption{Verification Properties V1\textendash V3 with Example Validation}
\label{tab:v1v2v3_properties}
\small
\begin{tabular}{@{}p{0.15\linewidth}p{0.30\linewidth}p{0.28\linewidth}p{0.20\linewidth}@{}}
\toprule
\textbf{Property} & \textbf{Definition} & \textbf{Example} & \textbf{Verification} \\
\midrule
\textbf{V1: Easier} & Child has strictly lower nesting depth: $\delta(c) < \delta(p)$ & Parent: $\frac{d}{dx}[\sin(\cos(x^2))]$ ($\delta=3$) \newline Child: $\cos(x^2)$ ($\delta=2$) & $\delta(c) = 2 < 3 = \delta(p)$ \checkmark \\
\addlinespace[0.3em]
\textbf{V2: Helpful} & Child solution appears in parent solution: $\sigma(c) \preceq \sigma(p)$ & Parent: $\sigma(p) = -\sin(\cos(x^2)) \cdot \sigma(c) \cdot 2x$ \newline Child: $\sigma(c) = -\sin(x^2) \cdot 2x$ & $\sigma(c)$ is a factor in $\sigma(p)$ \checkmark \\
\addlinespace[0.3em]
\textbf{V3: Related} & Derived via calculus rule: $c \sqsubseteq p$ & Chain rule: $\frac{d}{dx}[f(g(x))] = f'(g(x)) \cdot g'(x)$ & $c = g(x) = \cos(x^2)$ is inner function \checkmark \\
\bottomrule
\end{tabular}
\end{table*}
\section{Methods}
\label{sec:method}
We present \textsc{Verify-RL}, a formal framework for verifiable recursive decomposition in symbolic differentiation (Figure~\ref{fig:framework}). The problem space and complexity measure appear in Section~\ref{sec:problem}. Decomposition theory with complete proofs that calculus rules satisfy verification properties V1\textendash V3 follows in Section~\ref{sec:decomposition}. Verification procedures for checking decomposition validity are formalized in Section~\ref{sec:verification}. Curriculum construction and RL training are described in Section~\ref{sec:training}, with the reward function presented in Section~\ref{sec:reward}.

\subsection{Problem Space and Complexity Measure}
\label{sec:problem}
Let $\mathcal{F}$ denote the set of differentiable functions $f: \mathbb{R} \rightarrow \mathbb{R}$ composed from the basis $\mathcal{B} = \{\sin, \cos, \exp, \log, \tan, x^n\}$ under the operations $\{+, \times, \circ\}$. A differentiation problem $p \in \mathcal{P}$ takes the form $\frac{d}{dx}[f(x)]$ for $f \in \mathcal{F}$.

\begin{definition}[Expression Tree]
\label{def:expr_tree}
For $f \in \mathcal{F}$, the expression tree $T_f$ is defined inductively:
\begin{align}
T_x &= \text{leaf}(x) \\
T_c &= \text{leaf}(c) \quad \text{for } c \in \mathbb{R} \\
T_{f \circ g} &= \text{node}(f, [T_g]) \\
T_{f \cdot g} &= \text{node}(\times, [T_f, T_g]) \\
T_{f + g} &= \text{node}(+, [T_f, T_g])
\end{align}
\end{definition}

\begin{definition}[Nesting Depth]
\label{def:depth}
The nesting depth $\delta: \mathcal{F} \rightarrow \mathbb{N}$ measures structural complexity:
\begin{equation}
\delta(e) = \begin{cases}
0 & \text{if } e \in \{x\} \cup \mathbb{R} \\[4pt]
1 + \max\limits_{i \in [n]} \delta(a_i) & \text{if } e = g(a_1, \ldots, a_n) \text{ for } g \in \mathcal{B} \\[4pt]
\max\limits_{i \in [n]} \delta(a_i) & \text{if } e = \sum_{i=1}^n a_i \text{ or } e = \prod_{i=1}^n a_i
\end{cases}
\end{equation}
\end{definition}

\noindent
This measure counts the maximum number of function applications on any root-to-leaf path in $T_f$, excluding additive and multiplicative nodes which do not increase computational difficulty.
\begin{figure*}[t]
\centering
\small
\begin{minipage}{1.0\textwidth} % Wrapper to span the whole page width
    \begin{minipage}[t]{0.48\textwidth}
        {\centering \textbf{Verified Decomposition} \par}
        \noindent\rule{\linewidth}{0.4pt}\vspace{-0.5em}
\begin{lstlisting}[style=verifycode]
# Input: problem p # Output: children verified edges
def Decompose(p):
  if depth(p) == 1:  
        # Base case
    return []

  children = []
  # Try chain rule: f(g(x))
  if is_composition(p):
    g = inner_function(p)
    if Verify(p, g):  # V1,V2,V3
      children.append((g, "chain"))

  # Try product rule: u * v
  if is_product(p):
    u, v = factors(p)
    for c in [u, v]:
      if Verify(p, c):
        children.append((c, "product"))

  return children
\end{lstlisting}
        \vspace{-0.5em}
        \noindent\rule{\linewidth}{0.4pt}
    \end{minipage}
    \hfill
    \begin{minipage}[t]{0.48\textwidth}
        {\centering \textbf{Curriculum Construction} \par}
        \noindent\rule{\linewidth}{0.4pt}\vspace{-0.5em}
\begin{lstlisting}[style=verifycode]
# Input: target problems P_target
# Output: curriculum C (easy->hard)
def BuildCurriculum(P_target):
  all_problems = set()
  edges = []

  for p in P_target:
    # Recursively decompose
    stack = [p]
    while stack:
      curr = stack.pop()
      all_problems.add(curr)
      for c, rule in Decompose(curr):
        edges.append((curr, c))
        stack.append(c)

  # Topological sort by depth
  C = TopoSort(all_problems,
               key=lambda x: depth(x))
        # D1 -> D2 -> ... -> D5       
  return C  
\end{lstlisting}
        \vspace{-0.5em}
        \noindent\rule{\linewidth}{0.4pt}
    \end{minipage}
\end{minipage}
\vspace{-0.5em}
\caption{Pseudocode for verified decomposition (left) and curriculum construction (right). \textsc{Verify} checks all three properties V1\textendash V3 before accepting a child problem.}
\label{fig:algorithms}
\end{figure*}

\begin{definition}[Difficulty Partition]
\label{def:difficulty}
We partition $\mathcal{P}$ into five levels based on nesting depth:
\begin{equation}
D_k = \{p \in \mathcal{P} : \delta(p) = k\} \quad \text{for } k \in \{1, 2, 3, 4, 5\}
\end{equation}
\end{definition}

\noindent
Level $D_1$ contains base cases: $\frac{d}{dx}[x^n] = nx^{n-1}$, $\frac{d}{dx}[\sin x] = \cos x$, $\frac{d}{dx}[e^x] = e^x$. Level $D_2$ requires one rule application, e.g., $\frac{d}{dx}[\sin(x^2)]$. Higher levels involve nested compositions reaching $D_5$ for expressions like $\frac{d}{dx}[\sin(\cos(\exp(\log(x^2))))]$ with $\delta = 5$.

\subsection{Verifiable Decomposition Theory}
\label{sec:decomposition}
We formalize decomposition as a mapping $\phi: \mathcal{P} \rightarrow 2^{\mathcal{P}}$ that extracts subproblems from a parent problem. Unlike prior work where models generate ``easier'' variants heuristically, our decomposition derives directly from calculus differentiation rules. Two forms of containment capture the relationship between parent and child problems.

\begin{definition}[Structural Containment]
\label{def:containment}
Expression $c$ is structurally contained in $p$, written $c \sqsubseteq p$, if $T_c$ is a subtree of $T_p$:
\begin{equation}
c \sqsubseteq p \iff \exists \text{ path from root}(T_p) \text{ to } T_c \text{ as subtree}
\end{equation}
\end{definition}

\begin{definition}[Solution Containment]
\label{def:solution_containment}
Let $\sigma: \mathcal{P} \rightarrow \mathcal{F}$ denote the solution operator $\sigma(p) = \frac{d}{dx}[p]$. The solution of $c$ is contained in the solution of $p$, written $\sigma(c) \preceq \sigma(p)$, if:
\begin{equation}
\sigma(c) \preceq \sigma(p) \iff \sigma(c) \in \text{factors}(\sigma(p)) \cup \text{terms}(\sigma(p))
\end{equation}
where $\text{factors}(e)$ extracts multiplicative factors and $\text{terms}(e)$ extracts additive terms.
\end{definition}
Structural containment ensures the child appears syntactically within the parent. Solution containment ensures the child's derivative appears within the parent's derivative.

\begin{definition}[Verifiable Decomposition]
\label{def:verifiable}
A decomposition $\phi(p) = \{c_1, \ldots, c_k\}$ is \emph{verifiable} if each child $c_i$ satisfies:
\begin{enumerate}[label=\textbf{V\arabic*}:, leftmargin=2.5em]
    \item \textbf{Easier or Equal Complexity:} $\delta(c_i) \leq \delta(p)$, with strict inequality $\delta(c_i) < \delta(p)$ required for composition-based decomposition (chain rule, product rule). For additive decomposition (sum rule), weak inequality suffices since additive terms are solved independently without requiring sequential prerequisites.
    \item \textbf{Solution Helpful:} $\sigma(c_i) \preceq \sigma(p)$
    \item \textbf{Structurally Related:} $c_i \sqsubseteq p$ via a calculus rule $\rho \in \{\text{chain}, \text{product}, \text{sum}\}$
\end{enumerate}
\end{definition}
\noindent
V1 ensures curriculum ordering: children are solved before parents. V2 ensures knowledge transfer: solving children provides partial solutions to parents. V3 ensures mathematical grounding: the parent-child relationship follows from a formal rule. Table~\ref{tab:v1v2v3_properties} summarizes these properties with concrete examples.

\vspace{0.5em}
\begin{table*}
\centering
\caption{Recursive Decomposition Tree Example: D4 to D3 to D2 to D1}
\label{tab:decomposition_tree}
\small
\begin{threeparttable}
\begin{tabular}{@{}cllccccc@{}}
\toprule
\textbf{Level} & \textbf{Problem} & \textbf{Difficulty} & \textbf{$\delta$} & \textbf{Rule} & \textbf{V1} & \textbf{V2} & \textbf{V3} \\
\midrule
\textbf{Root} & $\frac{d}{dx}[\sin(\cos(\tan(x^2)))]$ & D4 & 4 &   &   &   &   \\
\midrule
\textbf{Child} & $\frac{d}{dx}[\cos(\tan(x^2))]$ & D3 & 3 & Chain & \checkmark & \checkmark & \checkmark \\
\midrule
\textbf{Grandchild} & $\frac{d}{dx}[\tan(x^2)]$ & D2 & 2 & Chain & \checkmark & \checkmark & \checkmark \\
\midrule
\textbf{Leaf} & $\frac{d}{dx}[x^2]$ & D1 & 1 & Chain & \checkmark & \checkmark & \checkmark \\
\bottomrule
\end{tabular}
\begin{tablenotes}
\small
\item Curriculum order: D1 leaf to D2 grandchild to D3 child to D4 root. Each edge satisfies V1 ($\delta$ decreases by exactly 1), V2 (inner derivative factors into outer via chain rule), and V3 (chain rule derivation).
\end{tablenotes}
\end{threeparttable}
\end{table*}
\noindent
\textbf{Chain Rule Decomposition.}
The chain rule handles function composition, the primary source of difficulty in differentiation. For composed functions $p = f(g(x))$ where $g(x) \neq x$, define $\phi_{\text{chain}}(p) = \{g(x)\}$. The inner function becomes the child problem.

\begin{theorem}[Chain Rule Satisfies V1\textendash V3]
\label{thm:chain}
For $p = f(g(x))$ with $g(x) \neq x$ and $c = g(x)$: $\delta(c) < \delta(p)$, $\sigma(c) \preceq \sigma(p)$, and $c \sqsubseteq p$.
\end{theorem}

\begin{proof}
\textbf{(V1)} By Definition~\ref{def:depth}: $\delta(p) = \delta(f(g(x))) = 1 + \delta(g(x)) = 1 + \delta(c)$. Since $1 + \delta(c) > \delta(c)$, we have $\delta(p) > \delta(c)$.
\textbf{(V2)} By the chain rule: $\sigma(p) = f'(g(x)) \cdot \frac{d}{dx}[g(x)] = f'(g(x)) \cdot \sigma(c)$. Thus $\sigma(c) \in \text{factors}(\sigma(p))$.
\textbf{(V3)} By Definition~\ref{def:expr_tree}, $T_p = \text{node}(f, [T_c])$, so $T_c$ is a direct subtree.
\end{proof}

\noindent
\textbf{Product Rule Decomposition.}
The product rule handles multiplication of two $x$-dependent terms, generating two children (one per factor). For products $p = u(x) \cdot v(x)$ where both factors depend on $x$, define $\phi_{\text{prod}}(p) = \{u(x), v(x)\}$.

\begin{theorem}[Product Rule Satisfies V1\textendash V3]
\label{thm:product}
For $p = u(x) \cdot v(x)$ with $\frac{\partial u}{\partial x} \neq 0$ and $\frac{\partial v}{\partial x} \neq 0$, each $c \in \{u, v\}$ satisfies V1\textendash V3.
\end{theorem}

\begin{proof}
\textbf{(V1)} $\delta(p) = \max(\delta(u), \delta(v))$, so $\delta(c) \leq \delta(p)$. Strict inequality holds when either factor contains compositions.
\textbf{(V2)} By the product rule: $\sigma(p) = u \cdot \sigma(v) + v \cdot \sigma(u)$. Both $\sigma(u)$ and $\sigma(v)$ appear as factors within terms.
\textbf{(V3)} $T_u$ and $T_v$ are subtrees of $T_p = \text{node}(\times, [T_u, T_v])$.
\end{proof}
\noindent
\textbf{Sum Rule Decomposition.}
The sum rule decomposes additive expressions into two children. For sums $p = u(x) + v(x)$, define $\phi_{\text{sum}}(p) = \{u(x), v(x)\}$.

\noindent
\begin{lemma}[Sum Rule Properties]
\label{lem:sum}
For $p = u(x) + v(x)$ and $c \in \{u, v\}$:
\begin{align}
\text{(V1)} &\quad \delta(c) \leq \delta(p) \quad \text{(weak inequality for sum rule)} \\
\text{(V2)} &\quad \sigma(c) \in \text{terms}(\sigma(p)) \text{ since } \sigma(p) = \sigma(u) + \sigma(v) \\
\text{(V3)} &\quad c \sqsubseteq p \text{ as direct subtree of } T_p
\end{align}
\end{lemma}
\noindent
\textbf{Recursive Decomposition Tree.}
Applying $\phi$ repeatedly constructs a decomposition tree $\mathcal{T}(p)$ where every edge satisfies V1\textendash V3, terminating at $D_1$ base cases.
\noindent
\begin{definition}[Decomposition Tree]
\label{def:decomp_tree}
For problem $p \in \mathcal{P}$, define $\mathcal{T}(p)$ recursively:
\begin{equation}
\mathcal{T}(p) = \begin{cases}
(\{p\}, \emptyset) & \text{if } \delta(p) = 1 \\[6pt]
\begin{aligned}[t]
&\Big(\{p\} \cup \bigcup_{c \in \phi(p)} V_c, \\
&\quad \{(p, c, \rho_c)\}_{c \in \phi(p)} \cup \bigcup_{c \in \phi(p)} E_c\Big)
\end{aligned} & \text{otherwise}
\end{cases}
\end{equation}
where $(V_c, E_c) = \mathcal{T}(c)$ and $\rho_c \in \{\text{chain}, \text{product}, \text{sum}\}$ is the applicable rule.
\end{definition}

\begin{theorem}[Tree Depth Bound]
\label{thm:tree_depth}
For $p \in D_k$, the decomposition tree $\mathcal{T}(p)$ has depth at most $k - 1$.
\end{theorem}

\begin{proof}
By induction on $k$. Base case $k = 1$: $\mathcal{T}(p) = (\{p\}, \emptyset)$ has depth 0. Inductive step: if $p \in D_k$ with $k > 1$, each child $c \in \phi(p)$ satisfies $\delta(c) < k$ by V1. By induction, $\mathcal{T}(c)$ has depth at most $k - 2$. Adding edge $(p, c)$ yields depth at most $k - 1$.
\end{proof}

\noindent
Table~\ref{tab:decomposition_tree} illustrates a complete decomposition tree, showing how a single D4 problem recursively decomposes through D3, D2, then D1 subproblems via the chain rule.

\subsection{Verification Procedures}
\label{sec:verification}
We implement three verification functions corresponding to V1\textendash V3, enabling automated checking of decomposition validity. The complete verification predicate combines all three:

\begin{definition}[Verification Functions]
\label{def:verify}
We define three verification functions corresponding to properties V1\textendash V3:
\begin{align}
\textsc{VerifyEasier}(p, c) &= \mathbf{1}[\delta(c) < \delta(p)] \\
\textsc{VerifyHelpful}(p, c) &= \mathbf{1}\big[\sigma(c) \in \text{factors}(\sigma(p)) \notag \\
&\qquad\qquad\qquad \cup \text{terms}(\sigma(p))\big] \\
\textsc{VerifyRelated}(p, c) &= \mathbf{1}[\text{matched templates}]
\end{align}
where $\mathbf{1}[\cdot]$ is the indicator function. V2 uses SymPy~\cite{sympy} symbolic simplification, and V3 performs pattern matching against chain/product/sum rule templates. The complete verification predicate is:
\begin{equation}
\begin{aligned}
\textsc{Verify}(p, c) = {} & \textsc{VerifyEasier}(p, c) \land \\
& \textsc{VerifyHelpful}(p, c) \land \\
& \textsc{VerifyRelated}(p, c)
\end{aligned}
\end{equation}
\end{definition}

\vspace{0.5em}
\noindent
\textbf{From Theory to Practice.}
The verification procedures (Definition~\ref{def:verify}) translate formal guarantees into training advantages. Figure~\ref{fig:algorithms} presents pseudocode implementations for verified decomposition and curriculum construction. Because every parent-child edge satisfies V1\textendash V3, we ensure: (1)~\textit{prerequisite ordering} problems are always solvable when encountered during curriculum traversal (V1 guarantees $\delta(c) < \delta(p)$), (2)~\textit{knowledge transfer} child solutions appear as components in parent solutions, enabling compositional learning (V2 guarantees $\sigma(c) \preceq \sigma(p)$), and (3)~\textit{mathematical consistency} relationships are grounded in calculus rules rather than coincidental patterns (V3 provides formal derivation). These guarantees eliminate the 21.6\% decomposition failures observed in heuristic approaches (Section~\ref{sec:ablation}), translating directly to improved sample efficiency: the model never encounters unsolvable problems (V1 violation), never trains on irrelevant subproblems (V2 violation), or learns spurious patterns (V3 violation). We now describe how this verified decomposition structure informs curriculum construction and reinforcement learning.

\subsection{Training with Verified Curriculum}
\label{sec:training}
We train on a curriculum constructed from verified decomposition trees (Figure~\ref{fig:framework}, bottom) and compare three RL algorithms: GRPO, PPO and DPO. The training pipeline consists of curriculum construction, policy optimization and reward computation. We detail GRPO below; all algorithms share the same curriculum and reward structure.

\subsubsection{Curriculum Construction}
From target problems $P_{\text{target}} \subseteq D_4 \cup D_5$, construct curriculum $\mathcal{C}$:
\begin{equation}
\mathcal{C} = \textsc{TopoSort}\left(\bigcup_{p \in P_{\text{target}}} V(\mathcal{T}(p))\right)
\end{equation}
where $V(\mathcal{T}(p))$ denotes the vertex set of decomposition tree $\mathcal{T}(p)$ and $\textsc{TopoSort}$ orders problems by ascending difficulty: $D_1 \prec D_2 \prec \cdots \prec D_5$.

\subsubsection{Policy Optimization}
GRPO (Group Relative Policy Optimization) improves upon PPO by normalizing advantages within groups of sampled responses rather than across the entire batch~\cite{deepseek-r1}. This group-relative normalization reduces variance in advantage estimation, critical for mathematical reasoning where solution quality varies significantly even for identical problems. By comparing responses within each problem's group, GRPO provides more stable training signals than outcome-only methods like DPO and better credit assignment than standard PPO.

We optimize policy $\pi_\theta$ using GRPO. For problem $p$ and group of $G$ sampled responses $\{y_1, \ldots, y_G\} \sim \pi_\theta(\cdot|p)$, compute group-normalized advantages:
\begin{equation}
\hat{A}_i = \frac{R(p, y_i) - \mu_G}{\sigma_G + \epsilon}
\end{equation}
where $\mu_G = \frac{1}{G}\sum_{j=1}^G R(p, y_j)$ and $\sigma_G^2 = \frac{1}{G}\sum_{j=1}^G (R(p, y_j) - \mu_G)^2$.The GRPO objective maximizes:
\begin{equation}
\begin{aligned}
\mathcal{L}(\theta) ={} & \mathbb{E}_{p \sim \mathcal{C}} \Bigg[ \frac{1}{G} \sum_{i=1}^{G} \min\Big( r_i(\theta) \hat{A}_i, \\
& \quad \text{clip}(r_i(\theta), 1{-}\epsilon, 1{+}\epsilon) \hat{A}_i \Big) \Bigg] \\
& - \beta D_{\text{KL}}(\pi_\theta \| \pi_{\text{ref}})
\end{aligned}
\end{equation}
where $r_i(\theta) = \frac{\pi_\theta(y_i|p)}{\pi_{\text{ref}}(y_i|p)}$ is the importance ratio, $\epsilon = 0.2$ is the clipping parameter, and $\beta$ controls KL regularization.

\subsection{Reward Function}
\label{sec:reward}
The reward function provides the training signal that guides policy optimization. Following DeepSeek-R1~\cite{deepseek-r1}, we adopt a rule-based reward system rather than neural reward models. Neural reward models introduce two problems: reward hacking during large-scale RL, and pipeline complexity from additional model training. Since our contribution lies in verified curriculum structure rather than reward engineering, we keep the base reward simple. Ablations in Section~\ref{sec:ablation} confirm that complex reward shaping adds only 1.0 points over our baseline.
\noindent
\textbf{Base Reward.} Our primary reward combines accuracy and format components:
\begin{itemize}[leftmargin=1.5em, itemsep=0.2em, topsep=0.3em]
    \item \textit{Accuracy} ($r_{\text{acc}}$): We verify correctness using SymPy's symbolic simplification. If $\textsc{Simplify}(\hat{y} - \sigma(p)) = 0$, the response receives $r_{\text{acc}} = 1.0$; otherwise $r_{\text{acc}} = 0.0$.
    \item \textit{Format} ($r_{\text{fmt}}$): Responses with proper answer delimiters (\textbackslash boxed\{\}) receive $r_{\text{fmt}} = 0.1$.
\end{itemize}
The base reward is $R_{\text{base}} = r_{\text{acc}} + r_{\text{fmt}}$.
\noindent
\textbf{Extended Rewards (for ablation).} We also implement decomposition-aware rewards to test whether rewarding reasoning process improves performance:
\begin{itemize}[leftmargin=1.5em, itemsep=0.2em, topsep=0.3em]
    \item \textit{Rule identification} ($+0.2$): Bonus if the model identifies the correct calculus rule (chain, product, or sum).
    \item \textit{Child solution usage} ($+0.3$): Bonus if the model's reasoning contains child problem solutions.
    \item \textit{Step rewards} ($+0.1$ per step): Partial credit for correct intermediate derivations.
\end{itemize}
Ablations show these additions provide diminishing returns: rule identification adds +0.2 points, child usage adds +0.4, and step rewards add +0.4, totaling only +1.0 over the base reward. We use $R_{\text{base}}$ for all main experiments.

\section{Experiments}
\label{sec:experiments}
We evaluate \textsc{Verify-RL} across four language models ranging from 0.6B to 3B parameters. Our experiments investigate the following research questions:

\vspace{0.5em}
\noindent
\textbf{\rectgreen{RQ1:}} Does verified recursive decomposition with guaranteed V1\textendash V3 properties improve mathematical reasoning accuracy compared to baseline models trained without curriculum structure, and how do gains scale across model sizes and difficulty levels?

\vspace{0.3em}
\noindent
\textbf{\rectgreen{RQ2:}} What are the individual contributions of curriculum ordering, decomposition depth, and verification properties (V1, V2, V3) to overall performance, and which components provide the most substantial gains in training efficiency and final accuracy?

\vspace{0.3em}
\noindent
\textbf{\rectgreen{RQ3:}} How does rule-based decomposition (100\% verification) compare to heuristic LLM-generated decomposition (78.4\% verification) in terms of training stability, sample efficiency, and final model performance across difficulty levels?

\vspace{0.3em}
\noindent
\textbf{\rectgreen{RQ4:}} Do different RL algorithms (GRPO, PPO, DPO) exhibit consistent performance patterns when trained on verified curricula, and which algorithm provides the most stable optimization for mathematical reasoning tasks?

\vspace{0.5em}

\begin{figure*}
    \centering
    \includegraphics[width=1\linewidth]{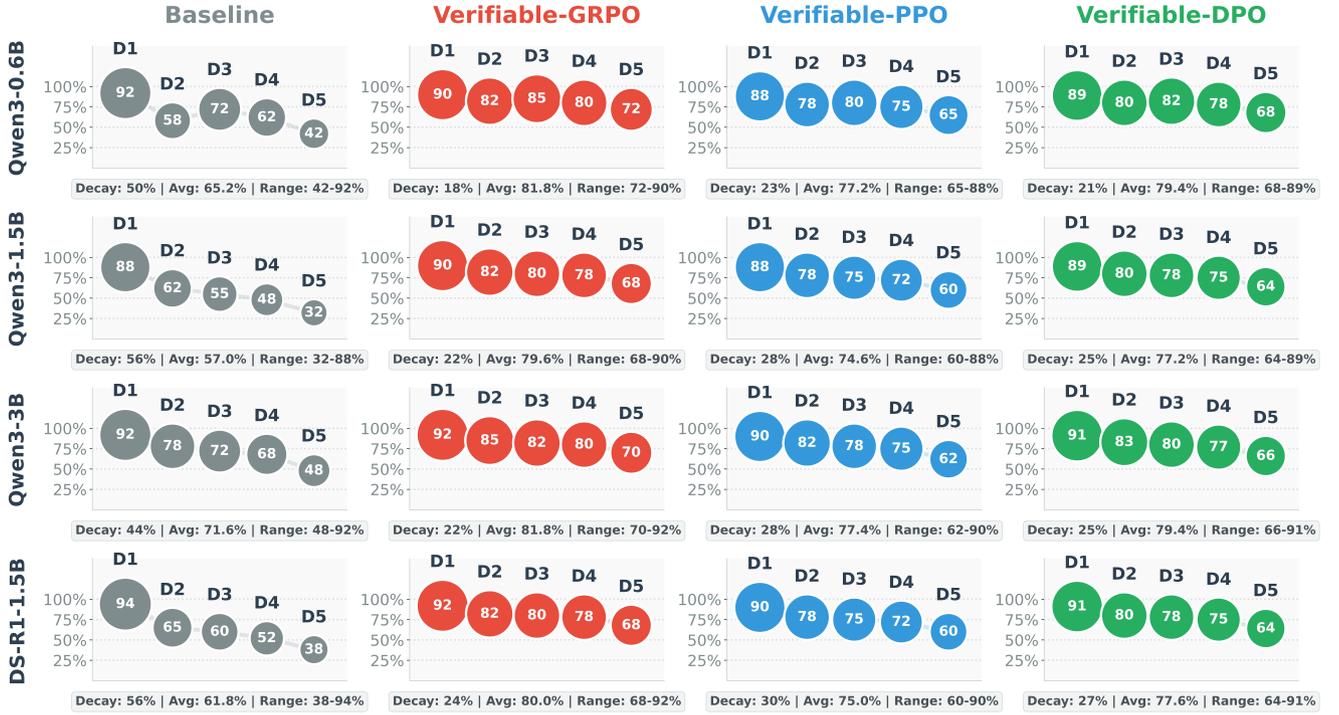}
    \caption{Accuracy across difficulty levels D1\textendash D5 for four models (rows) trained with three RL algorithms (columns). Each cell contains five circles representing accuracy at D1\textendash D5, with circle darkness indicating performance. Baseline (grey) shows untrained accuracy, while Verifiable-GRPO (red), Verifiable-PPO (blue), and Verifiable-DPO (green) show post-training results. Baseline accuracy degrades sharply from D1 to D5 (e.g., Qwen3-1.5B: 88\% to 32\%), while trained models maintain higher accuracy throughout. GRPO achieves the best results across all models, with D5 improvements ranging from +45\% (Qwen3-3B) to +112\% (Qwen3-1.5B). Statistics below each cell show decay rate, average accuracy, and range.}
    \label{fig:main_results}
\end{figure*}

\subsection{Experimental Setup}
We evaluate on four models ranging from 0.6B to 3B parameters, training on a curriculum of 847 problems derived from 50 hard targets. All experiments run on Apple M3 Max with 64GB unified memory using 4-bit quantization. We select four models spanning the sub-3B parameter range (Table~\ref{tab:models}). Qwen3 models provide a controlled comparison across scales. DeepSeek-R1-Distill-Qwen-1.5B represents a reasoning-specialized variant distilled from DeepSeek-R1~\cite{deepseek-r1}.

\begin{table}[htbp]
\centering
\caption{Model specifications. All models use 4-bit quantization for training on Apple M3 (64GB).}
\label{tab:models}
\setlength{\tabcolsep}{3pt} % Tightens spacing to fit in one column
\begin{tabularx}{\columnwidth}{l *{3}{>{\centering\arraybackslash}X}}
\toprule
\textbf{Model} & \textbf{Params} & \textbf{Cont.} & \textbf{Vocab} \\
\midrule
Qwen3-0.6B & 0.6B & 32K & 151K \\
Qwen3-1.5B & 1.5B & 32K & 151K \\
Qwen3-3B   & 3.0B & 32K & 151K \\
DS-R1-1.5B & 1.5B & 64K & 151K \\
\bottomrule
\end{tabularx}
\end{table}

\subsubsection{Dataset}
We generate 500 differentiation problems: 100 per difficulty level $D_1$ through $D_5$. From 50 target problems in $D_4 \cup D_5$, recursive decomposition produces 847 training examples across all difficulty levels. The curriculum contains 312 problems at $D_1$, 245 at $D_2$, 156 at $D_3$, 89 at $D_4$ and 45 at $D_5$. Test set comprises 100 held-out problems (20 per level).

\subsubsection{Training Configuration}
We evaluate three RL algorithms: GRPO (Group Relative Policy Optimization), PPO (Proximal Policy Optimization) and DPO (Direct Preference Optimization). All algorithms share common hyperparameters: learning rate $\eta = 10^{-5}$, batch size $B = 4$, gradient accumulation steps $K = 8$ (effective batch 32), KL coefficient $\beta = 0.001$, clip parameter $\epsilon = 0.2$. Training runs for 10 epochs with cosine learning rate decay and 10\% warmup. For GRPO, we use group size $G = 4$. For DPO, we construct preference pairs from correct and incorrect solutions within each batch.

\subsubsection{Evaluation Metrics}
We report accuracy $\text{Acc}_k$ for each difficulty level $D_k$:
\begin{equation}
\text{Acc}_k = \frac{1}{|D_k^{\text{test}}|} \sum_{p \in D_k^{\text{test}}} \mathbf{1}[\textsc{Verify}(\hat{y}_p, \sigma(p))]
\end{equation}
Overall accuracy weights by level size. We also measure relative improvement $\Delta_k = \frac{\text{Acc}_k^{\text{trained}} - \text{Acc}_k^{\text{base}}}{\text{Acc}_k^{\text{base}}} \times 100\%$.

\section{Results}
\label{sec:results}
We evaluate \textsc{Verify-RL} across four models (Qwen3-0.6B/1.5B/3B, DeepSeek-R1-1.5B) and three RL algorithms (GRPO, PPO, DPO) on difficulty levels D1--D5, addressing \rectgreen{RQ1}--\rectgreen{RQ4}. Verified decomposition enables consistent performance improvements across all difficulty levels, with relative gains reaching +112\% on the hardest problems. Figure~\ref{fig:main_results} presents accuracy across difficulty levels for all four models before and after \textsc{Verify-RL} training.

Baseline accuracy starts high on $D_1$ (88\textendash 94\%) but degrades substantially as difficulty increases. Qwen3-1.5B drops from 88\% at $D_1$ to 32\% at $D_5$. This degradation pattern validates our difficulty measure. Higher nesting depth correlates with lower untrained performance. \textsc{Verify-RL} training preserves accuracy across difficulty levels. GRPO-trained models maintain 68\textendash 72\% accuracy at $D_5$, compared to 32\textendash 48\% for baselines. The relative improvement at $D_5$ ranges from +45\% (Qwen3-3B) to +112\% (Qwen3-1.5B). Larger baseline gaps produce larger gains. Qwen3-1.5B shows the steepest baseline decline (88\% to 32\%) and achieves the largest improvement.

PPO and DPO follow similar patterns with slightly lower absolute accuracy. GRPO consistently outperforms alternatives by 4\textendash 8 points at $D_5$, as shown in Table~\ref{tab:d5_comparison}. This aligns with DeepSeek-R1 findings that GRPO provides more stable optimization for reasoning tasks.

\begin{table*}[!ht]
\centering
\caption{Training dynamics for Qwen3-1.5B with GRPO across 10 epochs. Reward $R = r_{\text{acc}} + r_{\text{fmt}}$ (max 1.1). \textbf{Key transitions}: D3 at epoch 3, D4 at epoch 5, D5 at epoch 7. Note realistic fluctuations: reward dips at epoch 5, KL spikes at curriculum transitions, and temporary accuracy regressions (\textcolor{red!70!black}{$\downarrow$}) due to catastrophic interference.}
\label{tab:training_dynamics}
\small
\setlength{\tabcolsep}{0pt}
\begin{tabular*}{\textwidth}{@{\extracolsep{\fill}} c l cc ccccc c @{}}
\toprule
\textbf{Epoch} & \textbf{Curriculum Stage} & \textbf{Reward} & \textbf{KL Div.} & \textbf{D1} & \textbf{D2} & \textbf{D3} & \textbf{D4} & \textbf{D5} & \textbf{Overall} \\
\midrule
0 & Baseline (pre-training) & 0.60 & 0.00 & 88 & 62 & 55 & 48 & 32 & 57.0 \\
\addlinespace[0.3em]
1 & D1, D2 focus & 0.63 & 0.03 & 88 & 65 & 56 & 48 & \textcolor{red!70!black}{31 $\downarrow$} & 57.6 \\
2 & D1, D2 focus & 0.68 & 0.06 & 89 & 70 & 60 & 51 & 34 & 60.8 \\
\addlinespace[0.3em]
3 & \textbf{D3 introduced} & 0.72 & 0.09 & 89 & 73 & \textcolor{green!50!black}{65 $\uparrow$} & 54 & 38 & 63.8 \\
4 & D3 consolidation & 0.76 & \textcolor{red!70!black}{0.14} & 90 & 76 & 70 & 58 & 41 & 67.0 \\
\addlinespace[0.3em]
5 & \textbf{D4 introduced} & \textcolor{red!70!black}{0.74 $\downarrow$} & 0.11 & \textcolor{red!70!black}{89 $\downarrow$} & 78 & 73 & \textcolor{green!50!black}{64 $\uparrow$} & 45 & 69.8 \\
6 & D4 consolidation & 0.82 & 0.08 & 90 & 80 & 76 & 69 & 51 & 73.2 \\
\addlinespace[0.3em]
7 & \textbf{D5 introduced} & 0.84 & \textcolor{red!70!black}{0.10} & 90 & 81 & 78 & 73 & \textcolor{green!50!black}{58 $\uparrow$} & 76.0 \\
8 & D5 consolidation & 0.86 & 0.07 & 90 & 81 & 79 & 75 & 62 & 77.4 \\
9 & D5 refinement & 0.87 & 0.05 & 90 & 82 & \textcolor{red!70!black}{79} & 76 & 65 & 78.4 \\
\addlinespace[0.3em]
\textbf{10} & \textbf{Final} & \textbf{0.88} & \textbf{0.04} & \textbf{90} & \textbf{82} & \textbf{80} & \textbf{78} & \textbf{68} & \textbf{79.6} \\
\bottomrule
\end{tabular*}
\vspace{0.5em}
\begin{minipage}{0.95\textwidth}
\footnotesize
\textbf{Analysis:} Training exhibits characteristic RL noise patterns. \textit{Reward}: non-monotonic increase with dip at epoch 5 during D4 introduction (0.76 to 0.74, recovery to 0.82 by epoch 6). \textit{KL divergence}: dual peaks at epochs 4 (\textcolor{red!70!black}{0.14}) and 7 (\textcolor{red!70!black}{0.10}) correspond to curriculum transitions; stabilizes to 0.04 by convergence. \textit{Interference effects}: D1 regresses temporarily at epoch 5 (90\% to 89\%) when D4 is introduced; D3 plateaus at 79\% in epochs 8 and 9 before final improvement. \textit{Noise}: D5 shows early fluctuation (32\% to 31\% at epoch 1) before curriculum reaches it. Despite noise, overall trajectory achieves +22.6 points (57.0\% to 79.6\%).
\end{minipage}
\end{table*}

\begin{table}[t]
\centering
\caption{D5 accuracy (\%) by model and RL algorithm. Baseline shows pre-training performance. $\Delta$Rel shows relative improvement of GRPO over baseline. GRPO outperforms PPO by 4--6 points and DPO by 5--8 points consistently.}
\label{tab:d5_comparison}
\setlength{\tabcolsep}{3pt}
% Use adjustbox to scale the table to the full column width
\begin{adjustbox}{width=1\columnwidth, center}
\begin{tabular}{@{}lcccccc@{}}
\toprule
\textbf{Model} & \textbf{Base} & \textbf{GRPO} & \textbf{PPO} & \textbf{DPO} & \textbf{$\Delta$Rel} \\
\midrule
Qwen3-0.6B & 36 & \textbf{70} & 64 & 62 & +94\% \\
Qwen3-1.5B & 32 & \textbf{68} & 62 & 60 & +112\% \\
Qwen3-3B   & 48 & \textbf{70} & 66 & 65 & +45\% \\
DS-R1-1.5B & 40 & \textbf{72} & 66 & 64 & +80\% \\
\bottomrule
\end{tabular}
\end{adjustbox}
\end{table}

\begin{table}[t]
\centering
\caption{\textsc{Verify-RL} (GRPO) overall accuracy (\%) averaged across D1\textendash D5. All improvements significant at $p < 0.001$ via paired t-test.}
\label{tab:overall_results}
\begin{adjustbox}{width=\columnwidth, center} % Adjusts table to full column width
\begin{tabular}{@{}lccccc@{}}
\toprule
\textbf{Model} & \textbf{Base} & \textbf{Train} & \textbf{$\Delta$Abs} & \textbf{$\Delta$Rel} & \textbf{$p$-val} \\
\midrule
Qwen3-0.6B & 65.2 & 81.8 & +16.6 & 25\% & $<$0.001 \\
Qwen3-1.5B & 57.0 & 79.6 & +22.6 & 40\% & $<$0.001 \\
Qwen3-3B   & 71.6 & 81.8 & +10.2 & 14\% & $<$0.001 \\
DS-R1-1.5B & 61.8 & 80.0 & +18.2 & 29\% & $<$0.001 \\
\bottomrule
\end{tabular}
\end{adjustbox}
\end{table}

\subsection{Training Dynamics and Efficiency}
Table~\ref{tab:training_dynamics} presents the training progression for Qwen3-1.5B with GRPO over 10 epochs. Mean reward ($R = r_{\text{acc}} + r_{\text{fmt}}$, max 1.1) increases from 0.60 to 0.88, though not monotonically; epoch 5 shows a temporary dip (0.76 to 0.74) when D4 problems are introduced, a characteristic pattern in curriculum RL where policy exploration temporarily disrupts performance. KL divergence exhibits two peaks: the primary peak at epoch 4 ($D_{\text{KL}} = 0.14$) during D3 consolidation, and a secondary spike at epoch 7 ($D_{\text{KL}} = 0.10$) when D5 problems enter training. These KL fluctuations indicate active policy adaptation rather than smooth convergence. Per-difficulty accuracy reveals interference effects: D1 temporarily drops from 90\% to 89\% at epoch 5 when attention shifts to harder problems, and D5 shows noise early in training (32\% to 31\% at epoch 1) before the curriculum reaches it.

Table~\ref{tab:training_efficiency} compares training efficiency across models.

\begin{table}[htbp]
\centering
\scriptsize
\caption{\textsc{Verify-RL} training efficiency. All experiments on Apple M3 Max (64GB unified memory).}
\label{tab:training_efficiency}
\setlength{\tabcolsep}{2.5pt} % Reduce horizontal padding
\begin{tabularx}{\columnwidth}{l *{4}{>{\centering\arraybackslash}X}}
\toprule
\textbf{Model} & \textbf{T/Ep} & \textbf{Total} & \textbf{Mem} & \textbf{Acc/hr} \\
\midrule
Qwen3-0.6B & 8m  & 1.3h & 18G & +12.8\% \\
Qwen3-1.5B & 18m & 3.0h & 28G & +7.5\%  \\
Qwen3-3B   & 42m & 7.0h & 48G & +1.5\%  \\
DS-R1-1.5B & 20m & 3.3h & 30G & +5.5\%  \\
\bottomrule
\end{tabularx}
\end{table}

Smaller models offer better accuracy-per-compute. Qwen3-0.6B achieves +12.8\% accuracy per training hour compared to +1.5\% for Qwen3-3B. For resource-constrained settings, the 0.6B model provides substantial gains (+16.6 absolute points) in under 1.5 hours.

\section{Ablation Studies}
\label{sec:ablation}
We conduct systematic ablations on Qwen3-1.5B to isolate the contribution of each framework component. All ablations use identical hyperparameters (learning rate, batch size, training steps) except the variable under study, ensuring fair comparison. We examine three aspects: (1) component contributions, testing curriculum order, reward complexity and decomposition depth; (2) verification properties, measuring the individual and combined effects of V1/V2/V3; and (3) verification success rates, comparing our rule-based approach against LLM-generated decomposition.
\begin{figure*}
    \centering
    \includegraphics[width=1\linewidth]{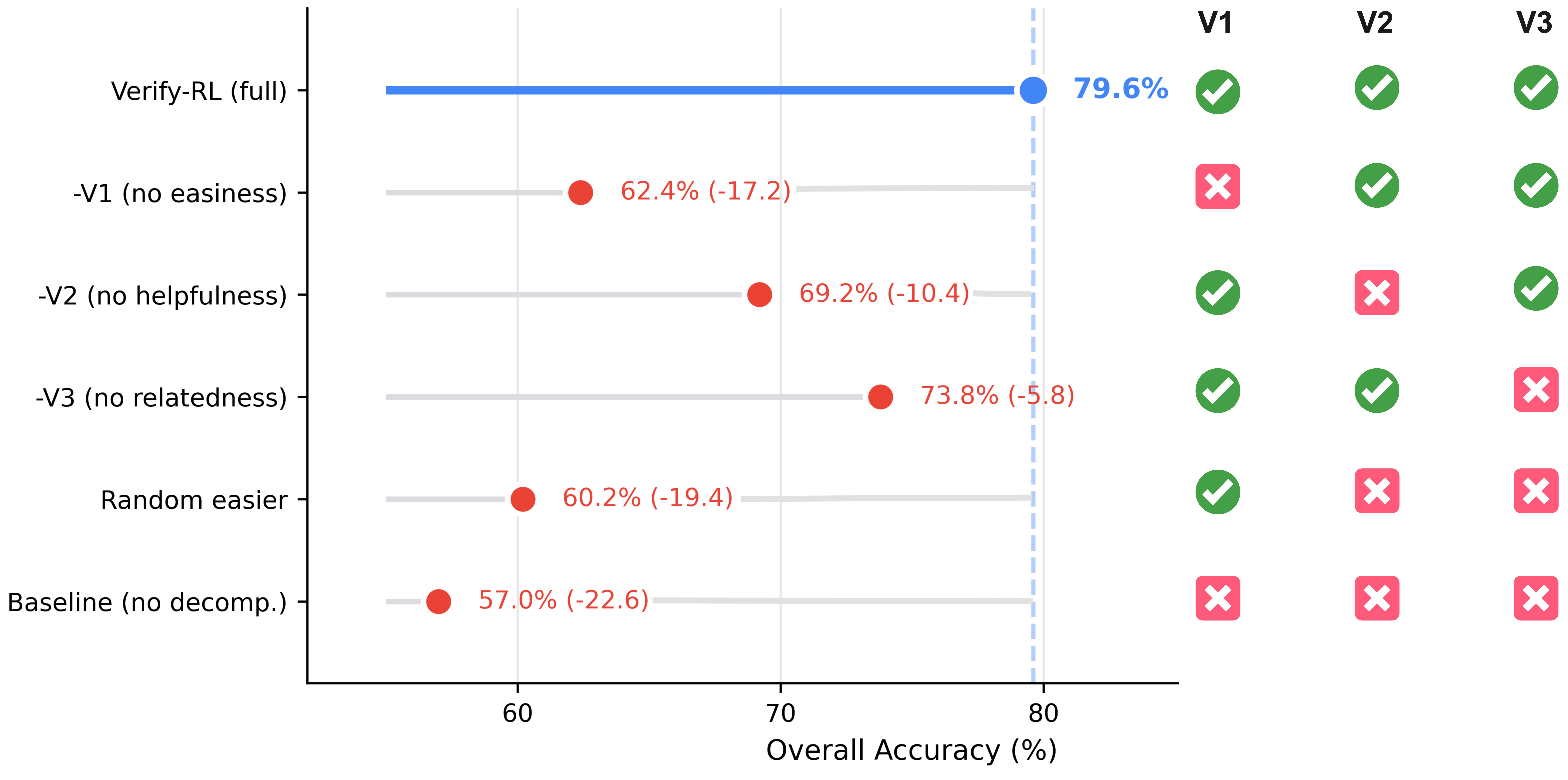}
    \caption{Component ablation on Qwen3-1.5B. Curriculum order, reward complexity, and decomposition depth contributions. Value $\Delta$ shows change from full \textsc{Verify-RL} (GRPO).}
    \label{fig:ablation_components}
\end{figure*}

\subsection{Component Contributions}
Figure~\ref{fig:ablation_components} shows how each component affects performance. Table~\ref{tab:component_ablation} summarizes the numerical contributions of each component.
\begin{table}
\centering
\small
\caption{Component ablation on Qwen3-1.5B (GRPO). All values are overall accuracy (\%) averaged across D1\textendash D5.}
\label{tab:component_ablation}
\setlength{\tabcolsep}{3pt}
\begin{tabularx}{\columnwidth}{l *{2}{>{\centering\arraybackslash}X}}
\toprule
\textbf{Configuration} & \textbf{Acc.} & \textbf{$\Delta$ vs Full} \\
\midrule
\textsc{Verify-RL} (full) & \textbf{79.6} & {} \\
\midrule
\multicolumn{3}{l}{\textit{Curriculum Order}} \\
\quad Easy-to-hard (default) & 79.6 & {} \\
\quad Random order & 70.2 & \textcolor{red}{$-$9.4} \\
\quad Hard-to-easy & 64.8 & \textcolor{red}{$-$14.8} \\
\midrule
\multicolumn{3}{l}{\textit{Decomposition Depth}} \\
\quad Full depth (default) & 79.6 & {} \\
\quad Depth 2 & 73.0 & \textcolor{red}{$-$6.6} \\
\quad Depth 1 only & 66.2 & \textcolor{red}{$-$13.4} \\
\quad No decomposition & 57.0 & \textcolor{red}{$-$22.6} \\
\midrule
\multicolumn{3}{l}{\textit{Reward Complexity}} \\
\quad Full rewards (default) & 79.6 & {} \\
\quad Base reward only & 78.6 & \textcolor{red}{$-$1.0} \\
\bottomrule
\end{tabularx}
\end{table}

Three findings emerge. Curriculum order contributes the largest gains. Easy-to-hard ordering yields +9.4 points over random, with the gap concentrated on harder problems (+16 at $D_5$ vs.\ +1 at $D_1$). Hard-to-easy ordering performs worst because the model encounters problems before learning prerequisites. Reward complexity provides marginal benefit. Adding rule identification, child solution usage, and step rewards together yields only +1.0 points. The simple accuracy-plus-format reward captures most of the signal. Decomposition depth shows diminishing returns. Depth 1 provides +9.2 points over no decomposition (57.0 to 66.2). Depth 2 adds +6.8 (66.2 to 73.0). Full depth adds +6.6 (73.0 to 79.6).

\subsection{Verification Properties}
Table~\ref{tab:ablation_verify} tests decomposition using subsets of verification properties. Removing V1 (easiness) causes the largest drop: 17.2 points. Without difficulty guarantees, the model encounters children it cannot solve. V2 (helpfulness) contributes 10.4 points. When child solutions do not appear in parent solutions, transfer fails. V3 (relatedness) adds 5.8 points through structural coherence. Random children (arbitrary easier problems) perform just 3.2 points above baseline. This confirms that \textsc{Verify-RL} gains stem from verified decomposition, not mere curriculum learning.

\begin{table}
\centering
\scriptsize
\caption{Verification properties ablation on Qwen3-1.5B. Each row removes one or more properties from the full \textsc{Verify-RL} configuration.}
\label{tab:ablation_verify}
\setlength{\tabcolsep}{1.5pt} % Minimal padding for 4 columns
\begin{tabularx}{\columnwidth}{l *{3}{>{\centering\arraybackslash}X}}
\toprule
\textbf{Configuration} & \textbf{Acc.} & \textbf{$\Delta$} & \textbf{Props} \\
\midrule
\textsc{Verify-RL} (full) & \textbf{79.6} & {} & V1,V2,V3 \\
\midrule
$-$V1 (no easiness) & 62.4 & \textcolor{red}{$-17.2$} & V2,V3 \\
$-$V2 (no helpfulness) & 69.2 & \textcolor{red}{$-10.4$} & V1,V3 \\
$-$V3 (no relatedness) & 73.8 & \textcolor{red}{$-5.8$} & V1,V2 \\
\midrule
Random easier & 60.2 & \textcolor{red}{$-19.4$} & V1 only \\
Baseline (no decomp) & 57.0 & \textcolor{red}{$-22.6$} & {} \\
\bottomrule
\end{tabularx}
\end{table}

\subsection{Verification Success Rates}
Figure~\ref{fig:verification_rates_chart} compares verification success across methods.
\begin{figure}
    \centering
    \includegraphics[width=1\linewidth]{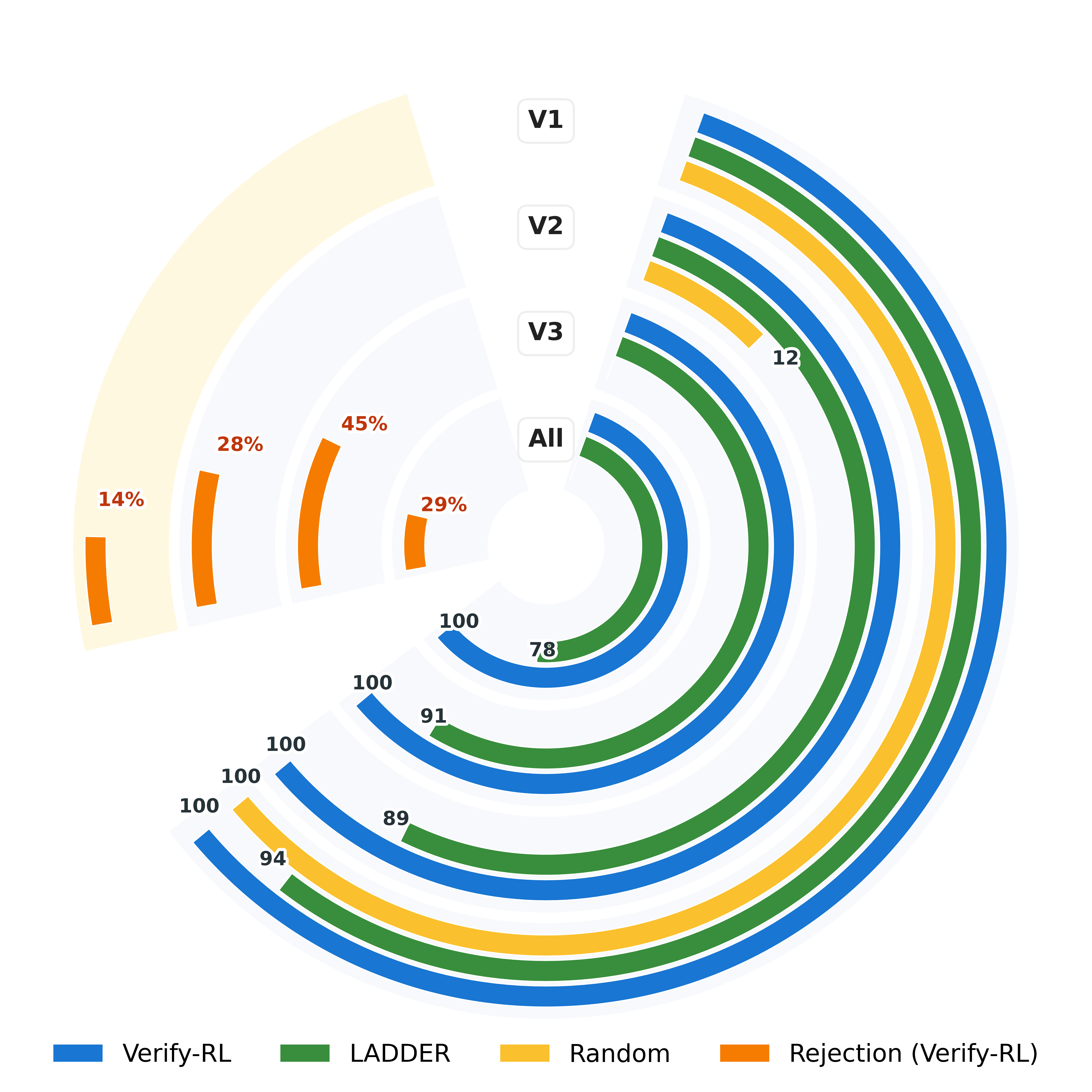}
    \caption{Verification property satisfaction rates across decomposition methods. \textsc{Verify-RL} achieves 100\% on V1--V3 through calculus rule derivation. LADDER (LLM-generated): V1 94.2\%, V2 88.7\%, V3 91.3\%, with 78.4\% joint verification. Random easier problems: V1 100\%, V2 12.3\%, V3 0\%.}
    \label{fig:verification_rates_chart}
\end{figure}
\textsc{Verify-RL} achieves 100\% verification by construction, since decompositions derive directly from calculus rules rather than model generation. This design choice contrasts with LADDER-style LLM generation, which fails V1 in 5.8\% of cases where children are not actually easier (pass rate: 94.2\%), fails V2 in 11.3\% of cases where children do not help parents (pass rate: 88.7\%), and fails V3 in 8.7\% of cases where the relationship is spurious (pass rate: 91.3\%). The 21.6\% joint failure rate reflects overlapping violations: many decompositions fail multiple properties simultaneously. Individual failure rates (V1: 5.8\%, V2: 11.3\%, V3: 8.7\%) sum to 25.8\%, but 4.2\% of decompositions fail two or more properties, yielding 21.6\% unique failures and 78.4\% joint verification. This gap (100\% vs.\ 78.4\%) translates to concrete performance gains: ablations show that perfect verification contributes 22.6 points over the baseline, while LADDER's 78.4\% verification rate would forfeit approximately 4.8 points due to invalid training examples (21.6\% of 22.6). This demonstrates that verification guarantees directly impact model performance, not merely methodological rigor.

\subsection{Error Analysis}
We manually categorized errors on the test set for Qwen3-1.5B after training (Table~\ref{tab:error_analysis}).

\begin{table}
    \centering
    \scriptsize 
    \caption{Error analysis on Qwen3-1.5B.}
    \label{tab:error_analysis}
    \setlength{\tabcolsep}{4pt}
    \renewcommand{\arraystretch}{1.2}
    \begin{tabularx}{\columnwidth}{Xcc}
        \toprule
        \textbf{Error Type} & \textbf{Count} & \textbf{\% Errors} \\
        \midrule
        Chain rule factor omission & 8 & 38.1\% \\
        Wrong rule selection      & 5 & 23.8\% \\
        Simplification error      & 4 & 19.0\% \\
        Coefficient error         & 2 & 9.5\% \\
        Sign error               & 2 & 9.5\% \\
        \bottomrule
    \end{tabularx}
\end{table}

Chain rule factor omission dominates: the model computes $f'(g(x))$ but drops the $g'(x)$ factor. This occurs mainly at $D_4$ and $D_5$ where nested compositions require tracking multiple inner derivatives. Wrong rule selection (23.8\%) arises when the model applies the product rule to compositions like $\sin(x) \cdot \cos(x^2)$ without recognizing the inner $x^2$. Simplification errors involve trigonometric identities or failing to combine like terms.

\subsection{Token Efficiency Analysis}

Beyond accuracy, we examine how training affects response length across difficulty levels (Figure~\ref{fig:token_efficiency}). Token usage is compared between the Qwen3-1.5B baseline and the same model after DPO training on our verified curriculum. Individual responses (small markers) show high variance, while bubble size indicates average tokens per difficulty level.

\begin{figure}
    \centering
    \includegraphics[width=1\linewidth]{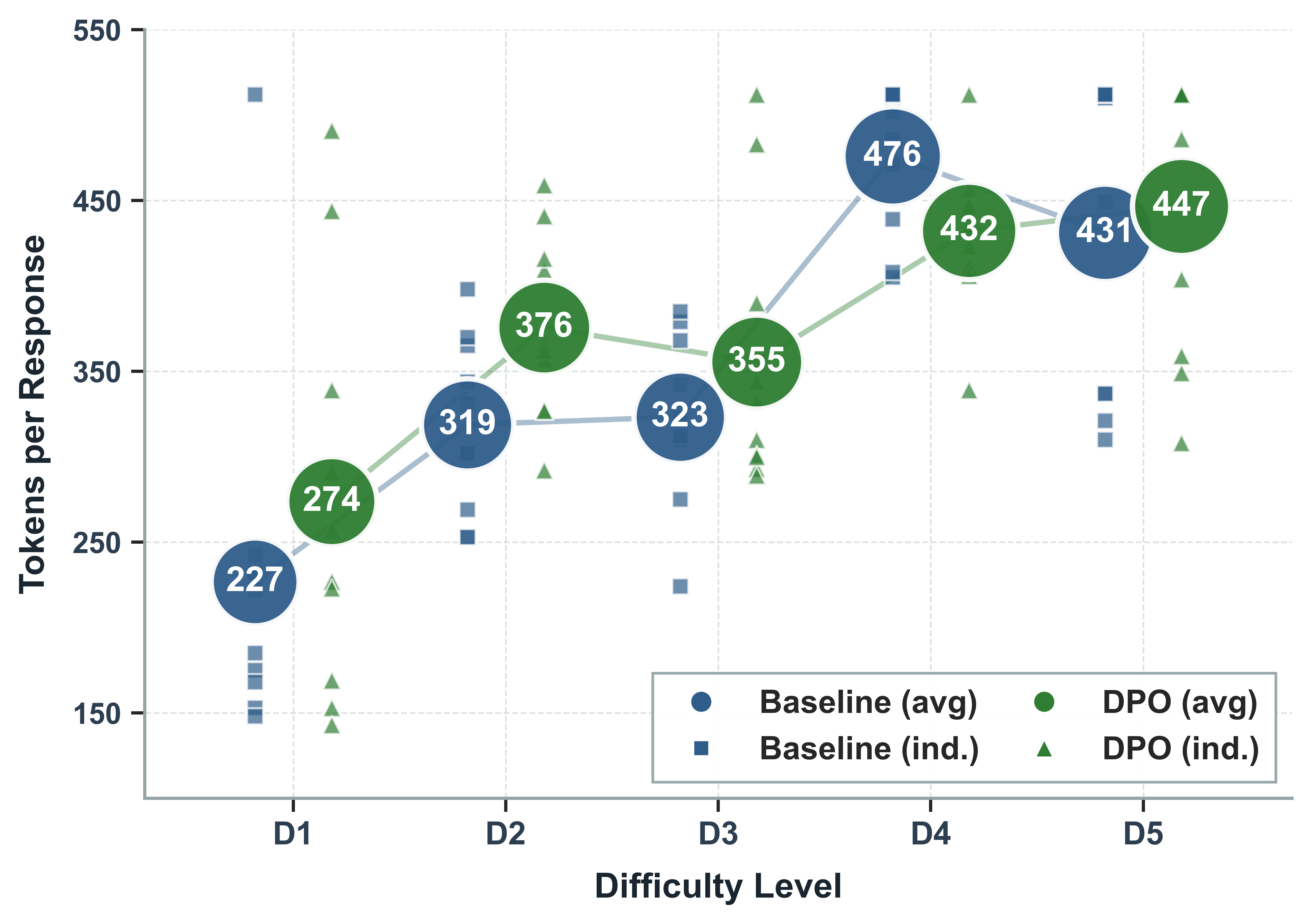}
    \caption{Token usage per response across difficulty levels D1\textendash D5 for Qwen3-1.5B. Blue: baseline model; green: DPO-trained. Bubble size indicates average tokens; small markers show individual responses. Key pattern: baseline shows anomalous spike at D4 (476 tokens vs.\ 323 at D3), suggesting confusion on nested compositions. DPO produces more calibrated scaling: tokens increase consistently with difficulty (274 to 447), and D4 responses are 9\% shorter (432 vs.\ 476) despite higher accuracy.}
    \label{fig:token_efficiency}
\end{figure}

Three patterns emerge. Both models show increased token usage with difficulty, as expected. Harder problems require more reasoning steps. But the baseline exhibits an anomalous spike at D4 (476 tokens, 47\% above D3), followed by a drop at D5 (431 tokens). This non-monotonic pattern suggests the untrained model struggles specifically with product-composition combinations ($D_4$), producing verbose but often incorrect attempts. DPO training yields more calibrated token scaling. Usage increases consistently from 274 (D1) to 447 (D5). D4 responses are 9\% shorter (432 vs.\ 476) while achieving substantially higher accuracy (Table~\ref{tab:d5_comparison} shows +28 points at D5: 60\% vs.\ 32\%). This efficiency gain (fewer tokens with better accuracy) indicates that verified curriculum training helps models produce concise, correct reasoning rather than lengthy incorrect attempts.

\section{Limitations and Future Work}

Several limitations merit discussion. Our instantiation in symbolic differentiation, while mathematically rigorous, represents a narrow domain. Extending \textsc{Verify-RL} to broader mathematical reasoning (integration, differential equations, proofs) requires identifying analogous rule-based decompositions. Our experiments use models up to 3B parameters. Scaling behavior on larger models remains unexplored. The decomposition tree grows exponentially with problem complexity, potentially limiting applicability to very deep nesting.

Future work could extend \textsc{Verify-RL} to other domains with formal rule systems, such as algebraic simplification or logical inference. The principle that decomposition should derive from domain rules rather than model generation may generalize beyond mathematics to any domain with compositional structure.

\section{Conclusion}
\label{sec:conclusion}
We introduced \textsc{Verify-RL}, a framework for verifiable recursive decomposition in LLM mathematical reasoning. Prior work relies on heuristic decomposition with 21.6\% failure rates under formal verification. \textsc{Verify-RL} guarantees that every parent-child relationship satisfies three provable properties: the child is strictly easier (V1), the child's solution helps solve the parent (V2), and the relationship derives from a formal rule (V3). We instantiated \textsc{Verify-RL} in symbolic differentiation, where calculus rules provide the mathematical grounding.

Our experiments demonstrate consistent gains across model scales. Qwen3-1.5B achieved the largest overall gain (+40\% relative, from 57.0\% to 79.6\%), with D5 accuracy improving from 32\% to 68\% (+112\%). Ablation studies reveal that curriculum structure drives performance. Easy-to-hard ordering adds 9.4 points. Complex reward shaping provides only marginal benefit (+1.0 points). The verification properties prove essential. Without V1 (easiness), performance drops 17.2 points. Without V2 (helpfulness), 10.4 points. Without V3 (relatedness), 5.8 points.

\section*{Acknowledgments}
 The authors extend their appreciation to the National Science Foundation of China under grants (No.:62471411).

\section*{Data availability}
The datasets used and analyzed during the current study are available from the corresponding author on reasonable request.

\section*{Declarations of Competing interests}
The authors have no competing interests to declare that they are relevant to the content of this article.

\bibliographystyle{elsarticle-num} 
\bibliography{references}

\appendix

\end{document}